\documentclass[letterpaper, 10 pt, conference]{ieeeconf}  

\IEEEoverridecommandlockouts                              
\usepackage{graphicx} 
\usepackage{times} 
\usepackage{graphicx}  
\usepackage{amssymb,amsmath}
\usepackage{cleveref}
\usepackage{tabularx}
\usepackage{booktabs}
\usepackage{multirow}

\newcommand{\R}{\mathbb{R}}

\newcommand{\N}{\mathcal{N}}
\newcommand{\dataset}{{\cal D}}

\newtheorem{theorem}{Theorem}

\newtheorem{proposition}{Proposition}

\usepackage[noadjust]{cite}

\title{\LARGE \bf Average Margin Regularization for Classifiers}

\author{Matt Olfat and Anil Aswani
\thanks{This material is based upon work partially supported by the National Science Foundation under
Grant CMMI-1847666, and by the UC Berkeley Center for Long-Term Cybersecurity.}
\thanks{Matt Olfat and Anil Aswani are with the Department of Industrial Engineering and Operations Research, University of California, Berkeley 94720}
        {\tt\small \{molfat,aaswani\}@berkeley.edu}}%

\begin{document}

	\maketitle
	\thispagestyle{empty}
	\pagestyle{empty}

\begin{abstract}
Adversarial robustness has become an important research topic given empirical demonstrations on the lack of robustness of deep neural networks. Unfortunately, recent theoretical results suggest that adversarial training induces a strict tradeoff between classification accuracy and adversarial robustness. In this paper, we propose and then study a new regularization for any margin classifier or deep neural network. We motivate this regularization by a novel generalization bound that shows a tradeoff in classifier accuracy between maximizing its margin and average margin. We thus call our approach an average margin (AM) regularization, and it consists of a linear term added to the objective. We theoretically show that for certain distributions AM regularization can both improve classifier accuracy and robustness to adversarial attacks. We conclude by using both synthetic and real data to empirically show that AM regularization can strictly improve both accuracy and robustness for support vector machine's (SVM's), relative to unregularized classifiers and adversarially trained classifiers.
\end{abstract}

\section{INTRODUCTION} \label{sec:intro}

There is renewed interest in robust learning due to the observed fragility of deep classifiers to imperceptible adversarial corruptions \cite{szegedy2013intriguing,fawzi2015fundamental,dalvi2004adversarial,biggio2017wild}. Such classifiers are often key elements in a variety of cyber-physical systems (CPS) like self-driving cars, smart homes, or smart grids. Adversarial perturbations on sensors within CPS can have disastrous consequences, and this has been exhibited by several major attacks on mission-critical CPS \cite{abrams2008malicious,langner2011stuxnet,cardenas2008research,ferrag2018security}. In response, numerous adversarial training approaches have been proposed, both in the context of linear margin classifiers \cite{lanckriet2002robust,trafalis2007robust,bertsimas2017robust} and deep classifiers \cite{goodfellow6572explaining,fawzi2015fundamental}. These approaches train classifiers so as to minimize loss with respect to adversarially-perturbed data.

Interestingly, such adversarial methods have been shown to be equivalent to a particular type of regularization in the context of linear margin classifiers and regression \cite{xu2009robust,xu2009robustness,livni2012simple,bertsimas2017robust}. While regularization protects against (though does not always eliminate) overfitting \cite{cortes1995support}, such outcomes critically depend upon having regularization that is congruous to the underlying data distributions \cite{aggarwal2001outlier,scholkopf2001learning,lanckriet2002robust,weston1999adaptive,bi2005support,fawzi2015fundamental,xu2009robustness}. 

Here, we argue that adversarial training ignores notable attributes of the data. For instance, image data often has manifold structure \cite{gerber2009manifold,pless2009survey,peyre2009manifold}. Similar claims may be made about the dynamics of learned systems in system identification or reinforcment learning contexts \cite{roberts1989utility,shi2010off}. Yet adversarial training regularizes with respect to full-dimensional perturbations and not with respect to any underlying manifold structure. This is significant because the imperceptibility of the most successful adversarial perturbations suggests that they lie orthogonal to these manifolds \cite{szegedy2013intriguing}. Thus, any robust methodology that does not exploit this kind of structure will likely remain susceptible to adversarial attacks.


Recently, \cite{tsipras2018there} have claimed that there is a strict trade-off between the accuracy of a classifier and its robustness to adversarial perturbations. They augment their argument with demonstrations of this inverse relationship on a specific dataset, and claim that adversarial training best minimizes the cost of robustness. This paper shows that this is trade-off is not general and that, in fact, robustness and accuracy can grow concurrently for broad classes of datasets. 

We make three contributions: First, we develop a novel generalization bound that shows classifier accuracy depends on a tradeoff between minimum and average margin. Second, we propose a new regularization that we call average margin (AM) regularization. This regularization consists of a linear term added to the objective, and is hence amenable to efficient numerical computation. We prove that for certain distributions, AM regularization can improve both accuracy and adversarial robustness of a classifier. Third, we use synthetic and real data to empirically show that AM regularization can generate support vector machine (SVM) classifiers that strictly dominate (in terms of accuracy and robustness) classifiers computed with or without adversarial training. Taken together, these results suggest that the phenomenon of adversarial fragility is an issue of overfitting rather than a fundamental issue unique to adversarial attacks.

\subsection{Robust Linear SVM} \label{subsec:litreviewrobustness}

Linear SVM relies upon on maximizing training margin by minimizing the \textit{hinge-loss}, and several methods have been proposed to improve its robustness. Some approaches truncate the hinge-loss function \cite{krause2004leveraging,collobert2006trading,liu2006multicategory,bartlett2002rademacher,suzumura2014outlier,yu2010relaxed}, but a major issue with this approach is that it forfeits the convexity of the original problem. Other methods \cite{song2002robust,masnadi2010design} penalize outliers uniformly while maintaining the convexity of the hinge-loss, and \cite{aggarwal2001outlier} generates sparse projections to minimize the visibility of outliers. Instead of attempting to devalue outliers, \cite{xu2006robust} formulate a mixed-integer problem with the hinge-loss that removes them, and \cite{weston1999adaptive} redesigns the loss function entirely using bounds on the leave-one-out cross-validation error. Adversarial training has also been considered: \cite{lanckriet2002robust,trafalis2007robust,bertsimas2017robust} take a minimax approach, solving bilevel programs to design classifiers robust to worst-case perturbations of either a given distribution or a given magnitude. Robustness of classifiers to noise in the label, as opposed to only the features, has also been considered \cite{bertsimas2017robust,biggio2011support,natarajan2013learning}.

\subsection{Robust Deep Classifiers}

Adversarial fragility is pronounced in deep classifiers. \cite{szegedy2013intriguing} identified the sensitivity of deep learners to adversarial noise, and \cite{goodfellow6572explaining,fawzi2015fundamental} followed up with empirical and theoretical examinations of this instability. It has been shown that relatively minute and visually unrecognizable perturbations (in the case of images) can significantly impact the accuracy of these learners, and some work has been done on characterizing these minimal deviations \cite{dalvi2004adversarial,biggio2017wild}. Notably, \cite{carlini2017adversarial} showed that many existing methods for adversarially-robust deep classification are not wholly effective. However, there has been a spate of promising recent results in this direction, some of which come with theoretical guarantees \cite{madry2017towards,kannan2018adversarial,raghunathan2018certified}.

\subsection{Outline} \label{subsec:outline}

Section \ref{sec:probsetting} describes our notation and presents the problem setup that will be considered.  Next, Section \ref{sec:method} introduces a novel generalization bound, proposes the average margin (AM) regularization, and then theoretically studies properties of this regularization for a specific distribution. Section \ref{sec:method} presents empirical results comparing linear classifiers that have been computed using different regularization and adversarial training approaches, using multiple datasets.

\section{NOTATION AND SETUP} \label{sec:probsetting}

We use $\N(\mu,\Sigma)$ to refer to a multivariate normal distribution with mean $\mu$ and covariance matrix $\Sigma$. Also, let $\mathbf{0}_d$ and $\mathbf{1}_d$ refer to vectors of length $d$ with all entries zero and one, respectively, and $I_d$ to the $d\times d$ identity matrix. These subscripts will be dropped when the size is obvious due to context. In contrast, the function notation $\mathbf{1}(\cdot)$ refers to the indicator function. We use $X_i$ to denote the $i$-th row of a matrix $X$. For some kernel function $k:\R^d\times\R^d\rightarrow\R$, let the matrix $K(X,X')$ be such that $K(X,X')_{ij}=k(X_i^{\vphantom{'}},X_j')$.

Consider data observations $(x,y)\sim\dataset$ where $x\in\R^d$ and $y\in\{+1,-1\}$. In binary classification, the goal is compute a classifier $h:\R^d\rightarrow\{+1,-1\}$ to predict a label $y$ from from a feature vector $x$. For a margin classifier from a family $\mathcal{H}$, this is achieved by minimizing $\min_{h\in\mathcal{H}}\textstyle\frac{1}{n}\sum_{i=1}^n\ell(y_ih(x_i))$, which is the sample average of some given loss function $\ell(\cdot)$. The expected classification error rate of a classifier $h$ is defined as $\mathcal{L}(h) = \mathbb{E}(\mathbf{1}(y \neq \mathrm{sign}(h(x))))$.

\emph{Adversarial robustness} for a classifier $h$ refers to its ability to maintain accuracy in predicting a label $y$ when given a corrupted corresponding feature vector $x+\delta$, where corruption $\delta$ is chosen by an adversary and has magnitude bounded by a quantity $e$. The expected adversarial classification error rate of a classifier $h$ when the adversary can perturb data by $e$ magnitude is $\mathcal{L}(h,e) = \mathbb{E}(\max_{\delta : \|\delta\|\leq e}\mathbf{1}(y \neq h(x+\delta)))$. The inner maximization is interpreted as an adversary choosing an attack. Also note that $\mathcal{L}(h,0) = \mathcal{L}(h)$.

\section{A NEW REGULARIZATION}\label{sec:method}

\begin{figure}[t!]
	\centering
	\includegraphics[clip, trim=0.5cm 1cm 0.5cm 1cm, width=1.0\linewidth]{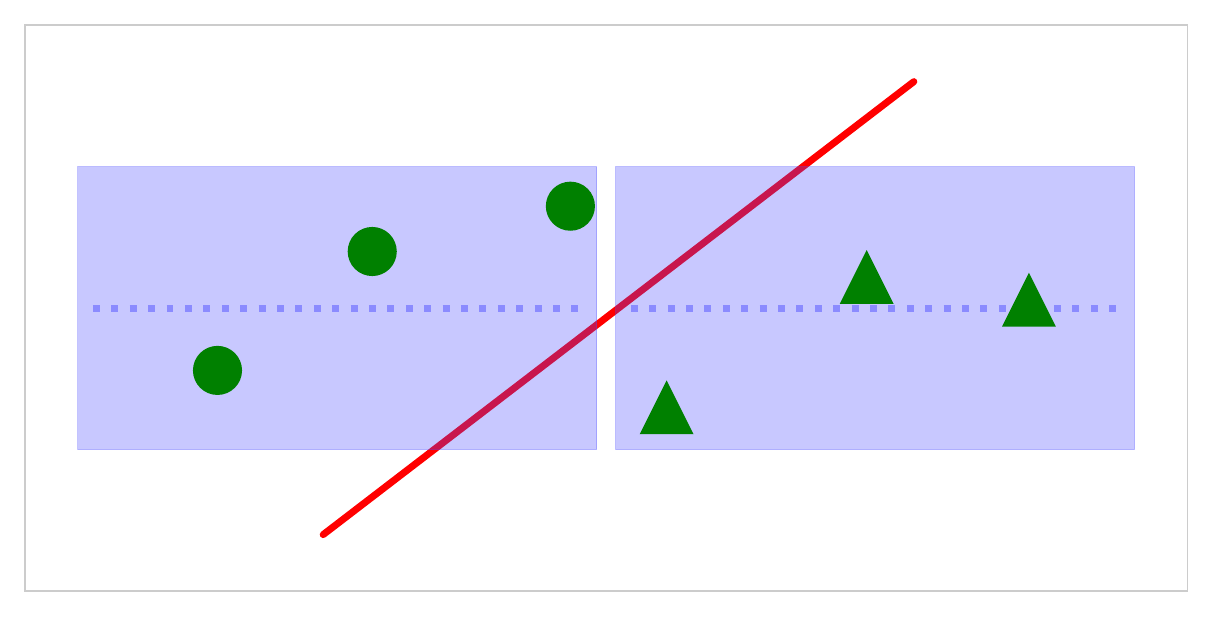}
	\caption{This example shows that maximizing the margin on training data can reduce classification accuracy since it does not use data far from the margin boundary. The marks are sampled data, the two supports of the data distributions for the two labels $y\in\{-1,+1\}$ are the two shaded rectangles, and the dashed line is the maximum margin linear classifier.}
	\label{fig:motivateplot}
\end{figure}

Margin classifiers primarily use only data near the boundaries of different classes for the purpose of estimating the parameters of the classifier. However, in the low (relative to dimensionality) data regime this can be problematic. \Cref{fig:motivateplot} shows an example where the usual margin classifier has issues. Maximizing the minimum margin leads to a classifier with high expected classification error because only a small amount of the data lies at the boundaries of the two classes. The manifold-like structure of the two classes leads to a situation where much of the data lies away from the boundary. A natural question to ask is how margin classifiers may be modified in order to better use data away from the boundary to improve predictions in situations similar to the above shown example.


Given the manifold-like example above, one possiblity is to use manifold regularization techniques. In fact, manifold regularization can be useful for regression \cite{belkin2006manifold,bickel2007local,aswani2011regression,cheng2013local}. However, a disadvantage of manifold regularization is it requires the indirect step of first estimating the manifold, and then using the estimated manifold for regularization. This indirect step can increase estimation error in a way that often reverses its regularizing effect. Our goal then is to design a regularizer that provides benefits in the manifold-like setting but does not require estimation of any manifolds.

In this section, we develop and study a new regularization for margin classifiers. We begin by proving a new generalization bound that demonstrates how maximizing the minimum margin does not always lead to minimal expected classification error. This generalization bound is used to motivate our new regularization for any margin classifier, which we call the average margin (AM) regularization. Next, we provide a probabilistic interpretation of this regularization in the context of deep learning. We conclude the section by discussing AM regularization in the special context of SVM. It is shown how this regularization can be used for kernel SVM, and then a result is given showing how AM regularization can simultaneously improve expected classification error and robustness to adversarial perturbations; this is significant because it is in direct contrast to results on adversarial training \cite{tsipras2018there} that find a strict trade-off between classifier accuracy and robustness to adversarial perturbations.

\subsection{Average Margin Generalization Bound}

Classical results on the generalization error of classifiers \cite{bartlett2002rademacher} provide justification for maximizing the minimum margin of classifiers. Below, we present a new generalization bound in terms of the average margin of a classifier.

\begin{theorem}\label{thm:amregbound}
	Let $\mathcal{L}(h) = \mathbb{E}(\mathbf{1}(y \neq \mathrm{sign}(h(x))))$ be the expected classification error rate of $h$, $K_\gamma(h) = \#\{i : y_ih(x_i) \leq \gamma\}/n$ be the fraction of data with $\gamma$-margin mistakes, $J(h) = \frac{1}{n}\sum_{i=1}^n y_i h(x_i)$ be the average classification margin, and suppose that $\sup_x |h(x)| \leq c$ for all $h\in\mathcal{H}$. Then for any $\zeta \in [0,1]$ we have with probability at least $1-2\delta$ that
	\begin{multline}
	\label{eqn:thmgpb}
	\mathcal{L}(h) \leq \zeta\cdot (1-J(h)/c) + (1-\zeta)\cdot K_\gamma(h)+\\
	4\frac{\mathcal{R}_n(\mathcal{H})}{\gamma} + \sqrt{\frac{\log(\log_2 \frac{4c}{\gamma})}{n}} + \sqrt{\frac{\vphantom{\log(\log_2\frac{4c}{\gamma}}\log(1/\delta)}{2n}}.
	\end{multline}
	for all $\gamma \in(0,c]$ and all $h\in\mathcal{H}$.
\end{theorem}

\begin{proof}
This proof uses a similar argument to the proof of Theorem 2 by \cite{kakade2009complexity}, with suitable modifications made to apply to our setting.  Define the functions
\begin{equation}
l_\gamma(u) = \begin{cases}
1, &u\leq 0\\
1 - u/\gamma, &0 < u < \gamma\\
0, & u \geq \gamma
\end{cases}
\end{equation}
and $\ell_\gamma(u) = \zeta\cdot(1 - u/c) + (1-\zeta)\cdot l_\gamma(u)$. Let $\mathcal{L}_\gamma(h) = \mathbb{E}(\ell_\gamma(yh(x)))$ and $\hat{\mathcal{L}}_\gamma(h) = \frac{1}{n}\sum_{i=1}^n\ell_\gamma(y_ih(x_i))$.  We will consider the values $\gamma_k = c/2^k$ and $\delta_k = \delta/(k+1)^2$ for $k \in \{0,1,\ldots\}$. Since $\ell_{\gamma_k}(u)$ is Lipschitz with constant $\zeta/ c + (1-\zeta)/\gamma_k \leq 1/\gamma_k$, applying Theorem 7 from \cite{bartlett2002rademacher} gives that
\begin{equation}
\mathcal{L}_{\gamma_k}(h) \leq \hat{\mathcal{L}}_{\gamma_k}(h) + \frac{2}{\gamma_k}\mathcal{R}_n(\mathcal{H}) + \sqrt{\frac{\log(1/\delta_k)}{2n}}
\end{equation}
holds with probability at least $1-\delta_k$ for all $h\in\mathcal{H}$. Next observe that for any $\zeta \in [0,1]$, $\gamma > 0$, and any $h\in\mathcal{H}$; we have $\mathcal{L}(h) \leq \mathcal{L}_\gamma(h)$ and $\hat{\mathcal{L}}_\gamma \leq \zeta\cdot (1-J(h)/c) + (1-\zeta)\cdot K_\gamma(h)$. Thus with probability at least $1-\delta_k$ we have
\begin{multline}
\label{eqn:onepf}
\mathcal{L}(h) \leq \zeta\cdot (1-J(h)/c) + (1-\zeta)\cdot K_{\gamma_k}(h) + \\\frac{2}{\gamma_k}\mathcal{R}_n(\mathcal{H}) + \sqrt{\frac{\log(1/\delta_k)}{2n}}
\end{multline}
for all $h\in\mathcal{H}$. Applying the union bound over all $k \in \{0,1,\ldots\}$ gives that (\ref{eqn:onepf}) holds with probability at least $1-\pi^2\delta/6 \geq 1-2\delta$ for all $k \in \{0,1,\ldots\}$ and $h\in\mathcal{H}$.  Now we will assume that this union event occurs.  This  implies (\ref{eqn:thmgpb}) holds for $\gamma = c$. Next consider $\gamma\in(0,c)$, and choose the $k$ such that $\gamma_k \leq \gamma<\gamma_{k-1}$. Note $k\leq \log_2(c/\gamma)+1$, and so 
\begin{multline}
\label{eqn:onepf1}
\mathcal{L}(h) \leq \zeta\cdot (1-J(h)/c) + (1-\zeta)\cdot K_{\gamma}(h) + \\\frac{4}{\gamma}\mathcal{R}_n(\mathcal{H}) + \sqrt{\frac{\log(1/\delta) + 2\log(\log_2(4c/\gamma))}{2n}}
\end{multline}
since $K_{\gamma_k}(h) \leq K_\gamma(h)$, $1/\gamma_k \leq 2/\gamma$, and $\log(1/\delta_k) \leq \log(1/\delta) + 2\log(\log_2(4c/\gamma))$.
\end{proof}


\subsection{Average Margin Regularization}

Given the above intuition, we propose that the average margin $\frac{1}{n}\sum_{i=1}^n y_i h(x_i)$ can be used as a regularization term for any margin classifier. Specifically, an AM-regularized classifier can be computed by solving 
\begin{equation}
\label{eq:generalrob}
\min_{h\in\mathcal{H}}\textstyle\frac{1}{n}\sum_{i=1}^n\ell(y_ih(x_i))-\mu\cdot\frac{1}{n}\sum_{i=1}^ny_ih(x),
\end{equation}
where $\mu\geq 0$ is a tuning parameter. Note that we subtract the average margin term because we are minimizing.

Next, we consider deep learning with activation function $\exp(h(x))/(1+\exp(h(x))$. The logistic loss is often used to construct classifiers, and the corresponding AM-regularized network is computed by solving
\begin{equation}
\min_{h\in\mathcal{H}}\textstyle\frac{1}{n}\sum_{i=1}^n\log(1+\exp(-y_ih(x_i))) - \mu\cdot\frac{1}{n}\sum_{i=1}^ny_ih(x_i),
\end{equation}
where $\mu\geq 0$ is again a tuning parameter. Now suppose we use the labels $t = (1+y)/2 \in\{0,1\}$, and we make the identifications that $\hat{p}_0(x) = 1/(1+\exp(h(x)))$ and $\hat{p}_1(x) = \exp(h(x))/(1+\exp(h(x)))$. Then training the AM-regularized network is equivalent to solving
\begin{multline}
\min_{h\in\mathcal{H}}\textstyle-\frac{1}{n}\sum_{i=1}^n\big[t_i\log(\hat{p}_1(x)) +
\textstyle(1-t_i)\log(\hat{p}_0(x))\big] + \\-\textstyle\mu\cdot\frac{1}{n}\sum_{i=1}^nt_i\log\frac{\hat{p}_1(x)}{\hat{p}_0(x)} +\\-\textstyle\mu\cdot\frac{1}{n}\sum_{i=1}^n(1-t_i)\log\frac{\hat{p}_0(x)}{\hat{p}_1(x)}.
\end{multline}
The first term is cross-entropy, while the last two terms are negative log-likelihood ratios. Thus, for deep learning our AM regularization encourages larger log-likelihood ratios. Recalling hypothesis testing, a larger log-likelihood ratio makes it easier to distinguish between classes.

\subsection{Special Case of SVM}

Here, we consider AM regularization in the special case of SVM. First, consider linear SVM with $h(x) = x^\textsf{T}\beta + b$.  Then the AM-regularized linear SVM is given by
\begin{equation}
\label{eqn:lsvm}
\begin{aligned}
\min\ & \textstyle\lambda\|\beta\|^2 + \frac{1}{n}\sum_{i=1}^n s_i - \mu\cdot\frac{1}{n}\sum_{i=1}^n y_i(x_i^\textsf{T}\beta + b)\\
\text{s.t. } & y_i(x_i^\textsf{T}\beta + b) \geq 1 - s_i, \quad\text{for } i = 1,\ldots,n\\
& s_i \geq 0, \hspace{2.54cm}\text{for } i = 1,\ldots,n
\end{aligned}
\end{equation}
As seen above, one advantage is that AM regularization is simply a linear term in the objective function. Thus, AM regularization does not significantly affect the computational complexity of solving the linear SVM optimization problem. Another benefit of AM regularization is it can be dualized, which allows for its use in kernel SVM. A standard argument using the KKT conditions shows that the kernel SVM with AM regularization is computed by solving
\begin{equation} \label{eq:dualrmsvm}
\begin{aligned}
\min\ & z^\textsf{T}\big(yy^\textsf{T}\circ K(X,X)\big)z - \mathbf{1}_nz^\textsf{T}\\
\text{s.t. } & yz^\textsf{T} = 0\\
&\textstyle\frac{\mu}{1+\mu}\cdot\frac{1}{\lambda n}\leq z\leq\frac{1}{\lambda n}
\end{aligned}
\end{equation}
This kernel SVM formulation provides further intuition about AM regularization: It shows that increasing $\mu$ increases the impact of points away from the margin, thereby mixing the original support vectors with an average over all data points.

Last, we show that AM regularization can generate classifiers that both improve classification accuracy and robustness. Let $\mathcal{L}(h,e) = \mathbb{E}(\max_{\delta : \|\delta\|\leq e}\mathbf{1}(y \neq h(x+\delta)))$ be the generalization error of classifier $h$ when the adversary can perturb data by $e$ magnitude. We specifically prove a result that formalizes the intuition shown in \Cref{fig:motivateplot}.

\begin{proposition}
\label{prop:bet}
Let $\hat{h}_{AM}$ and $\hat{h}_{L2}$ be the linear classifiers computed by SVM with and without AM regularization, respectively, using $n=4k\geq 4$ points sampled from a data distribution. Recalling (\ref{eqn:lsvm}), we assume $\lambda$ is chosen so that all sampled points lie beyond the margin, that $\mu$ can be chosen based on the sampled data, and that without loss of generality the linear classifier has no intercept term (i.e. $h(x) = x^\textsf{T}\beta$). There exists a data distribution such that 
\begin{align*}
\mathcal{L}(\hat{h}_{AM},e) < \mathcal{L}(\hat{h}_{L2},e), &\ \mathrm{for }\ e \in [0,7\sqrt{5}/25)\\
\mathcal{L}(\hat{h}_{AM},e) = \mathcal{L}(\hat{h}_{L2},e), &\ \mathrm{for }\ e \in [7\sqrt{5}/25, 15\sqrt{5}/25)\\
\mathcal{L}(\hat{h}_{AM},e) > \mathcal{L}(\hat{h}_{L2},e), &\ \mathrm{for }\ e \in [15\sqrt{5}/25, \sqrt{5})\\
\mathcal{L}(\hat{h}_{AM},e) = \mathcal{L}(\hat{h}_{L2},e), &\ \mathrm{for }\ e \in [\sqrt{5}, 2\sqrt{5})\\
\mathcal{L}(\hat{h}_{AM},e) < \mathcal{L}(\hat{h}_{L2},e), &\ \mathrm{for }\ e \in [2\sqrt{5}, 110\sqrt{5}/25)\\
\mathcal{L}(\hat{h}_{AM},e) = \mathcal{L}(\hat{h}_{L2},e), &\ \mathrm{for }\ e \in [110\sqrt{5}/25, \infty)\\
\end{align*}
\end{proposition}

\begin{proof}
We consider a balanced data distribution with $y_{2i-1} = +1$ and $y_{2i} = -1$ for $i=1,\ldots,2k$. Suppose $x_{2i-1} = (10,0)$ and $x_{2i} = (-10,0)$ for $i = k+1,\ldots,2k$. For $i = 1,\ldots,k$: let $a_i=+1$ be the event that $x_{2i-1} = (1,2)$ and $x_{2i} = (-1,-2)$, and let $a_i=-1$ be the event that $x_{2i-1} = (1,-2)$ and $x_{2i} = (-1,2)$. We assume the $a_i$ are independent Rademacher random variables.

Let $I = \{1,\ldots,k\}$, and note the margin assumption on $\lambda$ implies the classifier satisfies $y_ix_i^\textsf{T}\beta \geq 1$ for all $i=1,\ldots,n$. Define the conditional expectation $\mathcal{L}(h,e,\mathcal{A}) = \mathbb{E}[\max_{\delta : \|\delta\|\leq e}\mathbf{1}(y \neq h(x+\delta))|\mathcal{A}]$ for some event $\mathcal{A}$. Clearly, choosing
$\mu=0$ removes the effect of AM regularization and ensures that
$\mathcal{L}(\hat{h}_{AM},e,\mathcal{A}) \leq
\mathcal{L}(\hat{h}_{L2},e,\mathcal{A})$.

Now consider the event $\mathcal{B}$ where $a_i = +1$ for all $i\in I$. Here, the margin assumption means the classifier must satisfy $10\beta_1 \geq 1$, $\beta_1 + 2\beta_2 \geq 1$. A straightforward calculation gives that $\hat{h}_{L2}$ has $\hat{\beta}_{L2} = (\frac{1}{5},\frac{2}{5})$ and $\mathcal{L}(\hat{h}_{L2},e,\mathcal{B}) = \frac{1}{4} + \frac{1}{4}\mathbf{1}(e \geq \sqrt{5}) + \frac{1}{2}\mathbf{1}(e \geq 2\sqrt{5})$. The AM-regularized SVM is 
\begin{equation}
\begin{aligned}
\min\ & \lambda\|\beta\|^2 - 2\mu k\cdot(\beta_1 + 2\beta_2) - 2\mu k\cdot(10\beta_1)\\
\text{s.t. } & 10\beta_1 \geq 1\\
& \beta_1 + 2\beta_1 \geq 1
\end{aligned}
\end{equation}
Now suppose we choose the largest $\mu$ such that the optimal solution satisfies $10\beta_1 > 1$ and $\beta_1 + 2\beta_1 = 1$. Then it can be easily verified that this largest value is $\mu = \lambda/(15k)$, and so a straightforward calculation gives that $\hat{h}_{AM}$ has $\hat{\beta}_{AM} = (\frac{11}{15},\frac{2}{15})$ and $\mathcal{L}(\hat{h}_{AM},e,\mathcal{B}) = \frac{1}{4}\mathbf{1}(e \geq 7\sqrt{5}/25) + \frac{1}{4}\mathbf{1}(e \geq 15\sqrt{5}/25) + \frac{1}{2}\mathbf{1}(e \geq 110\sqrt{5}/25)$. This means that $\mathcal{L}(\hat{h}_{AM},e,\mathcal{B}) < \mathcal{L}(\hat{h}_{L2},e,\mathcal{B})$ for $e \in [0,\sqrt{5})$ and $\mathcal{L}(\hat{h}_{AM},e,\mathcal{B}) = \mathcal{L}(\hat{h}_{L2},e,\mathcal{B})$ for $e \geq \sqrt{5}$. And as discussed
earlier, choosing $\mu = 0$ for the event $\neg\mathcal{B}$ ensures
that $\mathcal{L}(\hat{h}_{AM},e,\neg\mathcal{B}) \leq
\mathcal{L}(\hat{h}_{L2},e,\neg\mathcal{B})$. Next note $\mathcal{B},\neg\mathcal{B}$ partition the sample space, and that $\mathbb{P}(\mathcal{B}) > 0$. Hence the result follows from the law of total expectation.
\end{proof}

This is a more subtle result than that of \cite{tsipras2018there}, which considers $\mathcal{L}(\cdot,e)$ for adversarially trained classifiers at only two discrete values of $e$. Our analysis shows that AM regularization can both increase classifier accuracy (i.e., at $e = 0$) and robustness to adversarial perturbations for specific data distributions. This is seen because AM regularization generally has lower expected classification error over the whole range of $e$ except for a very narrow range. The AM regularization is less robust when $e \in[15\sqrt{5}/25, \sqrt{5})$ because it has made a careful tradeoff between maximizing the margin and maximizing the average margin.

Another point to note is that for the data distribution in the above proposition, the adversarially trained SVM does not improve upon the standard SVM.

\begin{proposition}
Let $\hat{h}_{D\gamma}$ be the linear classifier computed by adversarially trained SVM where the adversary can perturb data by $\gamma \in[0,\sqrt{5})$ and let $\hat{h}_{L2}$ be the linear classifier computed by SVM, using $n=4k\geq 4$ points sampled from a data distribution. Recalling (\ref{eqn:lsvm}), we assume $\lambda$ is chosen so all sampled (and perturbed) points lie beyond the margin, and that without loss of generality the linear classifier has no intercept term (i.e. $h(x) = x^\textsf{T}\beta$). Then, for the distribution that is considered in the proof of \Cref{prop:bet}, we have that $\mathcal{L}(\hat{h}_{D\gamma},e) = \mathcal{L}(\hat{h}_{L2},e)$ for all $e \geq 0$.
\end{proposition}

\begin{proof}
The proof of this proposition follows a similar argument to that for Proposition \ref{prop:bet}.
\end{proof}

\begin{figure*}[!t]
	\begin{center}
	\makebox[\textwidth]{\includegraphics[width=1.2\linewidth]{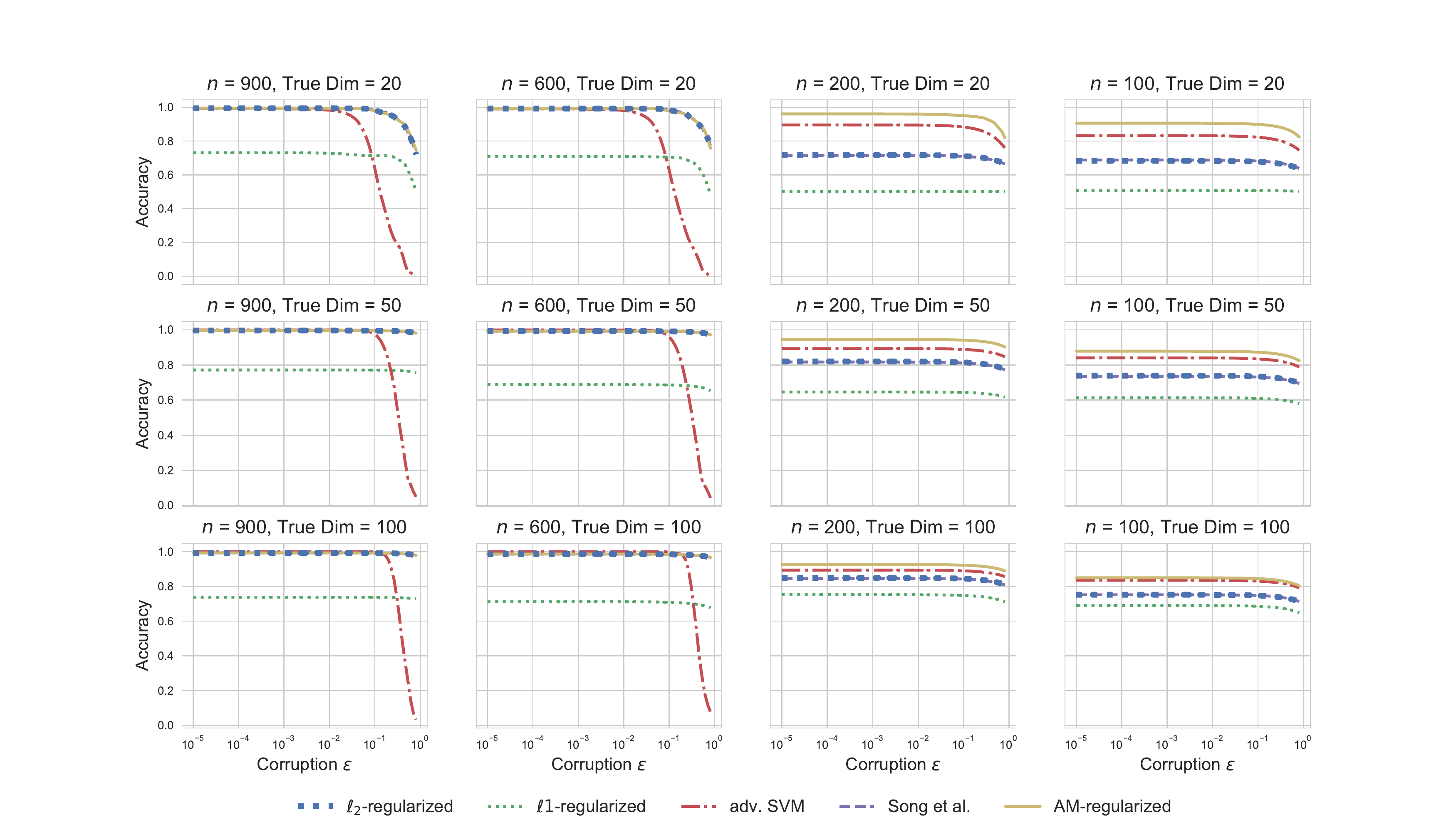}}
	\caption{Robustness to adversarial corruptions of classifiers trained on synthetic data, with accuracy measured as the fraction correctly classified.}
	\label{fig:randdata}
	\end{center}
\end{figure*}

\section{EMPIRICAL RESULTS} \label{sec:results}

Here, we use synthetic and real data to compare AM regularization to other regularizations, including adversarial training, for linear SVM. We first present the benchmark regularization methods that we compare AM regularization to. Next, we present empirical results for linear SVM using synthetic data, and we conclude this section by presenting numerical results for linear SVM applied to a real dataset.

\subsection{Benchmarks}

We compare linear SVM with AM regularization to: SVM with $\ell_2$ regularization, SVM with $\ell_1$ regularization, the robust SVM training method of \cite{song2002robust}, and the adversarial training method outlined by \cite{trafalis2007robust,xu2009robustness,bertsimas2017robust}.

\paragraph{Song et al.}

This method relies on minimizing the impact of outliers on the design of the separator. Specifically, it avoids the over-reliance on such outliers by shifting the loss according to the distance of a point to its respective centroid:
\begin{multline}
\min \textstyle\lambda\|\beta\|^2+\\\textstyle\frac{1}{n}\sum_{i=1}^n\left(1-y_i(\beta^{\textsf{T}}x_i^{\vphantom{\textsf{T}}}+b)-\gamma\|x_i-\mu_{y_i}\|_2^2\right)_+,
\end{multline}
where $\mu_{y_i}$ is the centroid for class $y_i$, and $\lambda,\gamma\geq 0$ are tuning parameters.

\paragraph{Adversarial Training}

Adversarial training formulates SVM training as a bilevel program, where the loss is minimized not over the original data $\{x_i\}_{i=1}^n$, but rather over worst-case perturbations of the data $\{x_i+\delta_i\}_{i=1}^n$, where the perturbations $\delta_i$ are restricted to be in some space. If we restrict $\|\delta_i\|_q\le\gamma$ for some norm $\|\cdot\|_q$, this is modeled as
\begin{equation} \label{eq:advsvm}
\begin{aligned}
\min\ &\textstyle\lambda\|\beta\|^2+\frac{1}{n}\sum_{i=1}^ns_i\\
\text{s.t. } & y_i(\beta^{\textsf{T}}x_i^{\vphantom{\textsf{T}}}+b)-\gamma\|\beta\|_{q^*}\geq 1-s_i, \text{ for } i = 1,\ldots,n\\
&s_i\geq 0,\hspace{3.93cm} \text{for } i = 1,\ldots,n
\end{aligned}
%
\end{equation}
where $\|\cdot\|_{q^*}$ is the dual norm of $\|\cdot\|_q$. In this section, we consider the case where $q=2$.

\subsection{Synthetic Data} \label{subsec:synthetic}

We next run randomized experiments to demonstrate situations where AM regularization is beneficial. We generate $U_+,U_-$ by taking the $Q$ from a QR decomposition done on a matrix of size $d\times m$ with standard normal entries, and $\mu_+,\mu_-\in\R^d$ to have uniformly-distributed entries between 0 and 2. Then, we set $x = \Pi_{U_+}\mu_++U_+v_++\varepsilon_+$ when $y=+1$, and $x = \Pi_{U_-}\mu_-+U_-v_-+\varepsilon_-$ when $y=-1$, where $v_+\sim\N(\mu_+,I_m)$, $\varepsilon_+\sim\N(\mathbf{0}_d,\varepsilon I_d)$, $v_-\sim\N(\mu_-,I_m)$, and $\varepsilon_-\sim\N(\mathbf{0}_d,\varepsilon I_d)$. A training set of $n$ samples was created. All models were then tested using 10,000 samples with adversarial corruption of various magnitude. This procedure was repeated 100 times with $\varepsilon=0.01$ and $d=200$, and the results for various training-set sizes and values of $n$ are presented in \Cref{fig:randdata}. Our method improves both robustness and accuracy in low-data and low-dimensionality regimes.

\subsection{MNIST Experiments} \label{subsec:mnist}

Next, we use the classic MNIST dataset \cite{lecun2010mnist} to evaluate AM regularization. In \Cref{tab:svm}, we compare AM-regularization to the benchmark methods on the fraction of test points correctly classified. For simplicity, we focus on the binary classification problem of separating hand-drawn 0's from hand-drawn 1's, and report results from training on 10\% of available training data. All hyperparameter values were set using 5-fold cross-validation on the training set. We repeat this process 50 times for each model, and report the accuracy results on a common testing set corrupted with various degrees of adversarial perturbations. We note that our method achieves better or almost equivalent accuracy with low corruption, and is best able to retain that level of accuracy at high levels of corruption.

\begin{table}[!t]
	\caption{Fraction Correctly Classified when Classifying 0 vs. 1 in MNIST Dataset}
	\label{tab:svm}
	\begin{center}
		\begin{small}
			\begin{sc}
				\begin{tabular}{lccc}
					\toprule[1.5pt]
					Corruption & 0.01 & 0.2 & 1.0 \\
					\midrule[1.5pt]
					$\ell_2$-regularized  &    0.940 &    0.938 &    0.921 \\
					$\ell_1$-regularized     &    0.997 &    0.995 &    0.714 \\
					AM-regularized     &    0.998 &    0.997 &    0.983 \\
					Adversarial   &    0.996 &    0.995 &    0.979 \\
					Song et al. &    0.994 &    0.992 &    0.974 \\
					\bottomrule[1.5pt]
				\end{tabular}
			\end{sc}
		\end{small}
	\end{center}
\end{table}

\section{CONCLUSION} \label{sec:conclusion}

Based on a novel generalization bound, we have proposed in this paper a new form of regularization for margin classifiers that we call average margin (AM) regularization. Our theoretical and empirical results support its use by showing that AM regularization can increase both classifier accuracy and adversarial robustness. Taken together, our results suggest that adversarial fragility is an issue of overfitting, rather than a fundamental uniqueness of adversarial perturbations. 

One future topic is to better understand AM regularization's theoretical properties. We believe AM regularization works best in the finite sample regime and that its asymptotic behavior when $\mu \not\rightarrow 0$ may be poor. Another topic is modifications of AM regularization. For instance, consider a hinged average margin (HAM) regularization that adds $-\mu\times \frac{1}{n}\sum_{i=1}^n(\gamma-y_ih(x_i))^+$, where $\mu,\gamma \geq 0$ are tuning parameters. HAM regularization may improve AM regularization by providing saturation, whereby points \emph{very} far (specifically $\gamma$ distance away) from the margin are not considered in the average margin calculation. Promisingly, this is convex in $y_ih(x_i)$. Furthermore, we would like to investigate the efficacy of AM regularization for deep classifiers.


\addtolength{\textheight}{-3cm}   




\bibliographystyle{IEEEtran}
\bibliography{IEEEabrv,rsvm}

\end{document}